\documentclass[runningheads]{llncs}
\usepackage{graphicx}
\usepackage{comment}
\usepackage{amsmath,amssymb} 
\usepackage{color}
\usepackage{array}
\usepackage{booktabs}
\usepackage[scr=boondox]{mathalfa}
\usepackage[mathcal]{euscript}
\usepackage{dsfont}
\usepackage{subfig}
\usepackage{makeidx}

\usepackage[hyphens]{url}
\usepackage{graphicx}
\usepackage{epsfig}
\usepackage{amsmath}
\usepackage{amssymb}
\usepackage{algorithm}
\usepackage{algpseudocode}
\usepackage{booktabs}
\usepackage{setspace}
\usepackage{nth}
\usepackage{dblfloatfix}
\usepackage{lipsum}
\usepackage{multicol}
\usepackage[export]{adjustbox}
\usepackage{textcomp}
\usepackage{graphicx}  

\usepackage{algcompatible}

\usepackage{amsmath,amsfonts,bm}

\def\eqref#1{equation~\ref{#1}}

\def\1{\bm{1}}

\newcommand{\test}{\mathcal{D_{\mathrm{test}}}}

\def\rvx{{\mathbf{x}}}

\def\rvz{{\mathbf{z}}}

\DeclareMathAlphabet{\mathsfit}{\encodingdefault}{\sfdefault}{m}{sl}
\SetMathAlphabet{\mathsfit}{bold}{\encodingdefault}{\sfdefault}{bx}{n}

\usepackage{pdfpages}

\usepackage[utf8]{inputenc}
\usepackage[english]{babel}
\def\rvz{{\mathbf{z}}}
\def\rvx{\mathbf{x}}

\usepackage{yfonts}

\usepackage[scr=boondox]{mathalfa}

\makeatletter
\newcommand{\printfnsymbol}[1]{%
  \textsuperscript{\@fnsymbol{#1}}%
}
\makeatother

\usepackage[noadjust]{cite}

\begin{document}
\pagestyle{headings}
\mainmatter
\def\ECCVSubNumber{6124}  
\title{Unsupervised Domain Adaptation for Semantic Segmentation of NIR Images through Generative Latent Search}
\titlerunning{UDA for Semantic Segmentation of NIR Images through GLS}
%
\author{Prashant Pandey\thanks{equal contribution}\orcidID{0000-0002-6594-9685} \and
Aayush Kumar Tyagi\index{Tyagi, Aayush Kumar}\printfnsymbol{1}\orcidID{0000-0002-3615-7283} \and
Sameer Ambekar\orcidID{0000-0002-8650-3180} \and \\ Prathosh AP\orcidID{0000-0002-8699-5760}}
%
\authorrunning{P. Pandey et al.}
\institute{Indian Institute of Technology Delhi \\
\email{getprashant57@gmail.com, aayush16081@iiitd.ac.in,
ambekarsameer@gmail.com,
prathoshap@iitd.ac.in}}
\maketitle
\begin{abstract}
Segmentation of the pixels corresponding to human skin is an essential first step in multiple applications ranging from surveillance to heart-rate estimation from remote-photoplethysmography. However, the existing literature considers the problem only in the visible-range of the EM-spectrum which limits their utility in low or no light settings where the criticality of the application is higher. To alleviate this problem, we consider the problem of skin segmentation from the Near-infrared images. However, Deep learning based state-of-the-art segmentation techniques demands large amounts of labelled data that is unavailable for the current problem. Therefore we cast the skin segmentation problem as that of target-independent Unsupervised Domain Adaptation (UDA) where we use the data from the Red-channel of the visible-range to develop skin segmentation algorithm on NIR images. We propose a method for target-independent segmentation where the `nearest-clone' of a target image in the source domain is searched and used as a proxy in the segmentation network trained only on the source domain. We prove the existence of `nearest-clone' and propose a method to find it through an  optimization algorithm  over the latent space of a Deep generative model based on variational inference. We demonstrate the efficacy of the proposed method for NIR skin segmentation over  the  state-of-the-art UDA  segmentation  methods on  the two newly  created skin segmentation datasets in NIR domain despite not having access to the target NIR data. Additionally, we report state-of-the-art results for adaption from Synthia to Cityscapes  which is a popular setting in Unsupervised Domain Adaptation for semantic segmentation. The code and datasets are available at https://github.com/ambekarsameer96/GLSS.

\end{abstract}

\section{Introduction}
\subsection{Background}
Human skin segmentation is the task of finding pixels corresponding to skin from images or videos. It serves as a necessary pre-processing step for multiple applications like video surveillance,
people tracking, human computer interaction, face detection and recognition, facial gesture detection and monitoring heart rate and respiratory rate \cite{prathosh2017estimation,mahmoodi2017high,chen2016skin,mahmoodi2016comprehensive} using remote photoplethysmography.  Most of the research efforts on skin detection have focused on visible spectrum images because of the challenges that it poses including, illumination change, ethnicity change and presence of background/clothes similar to skin colour. These factors adversely affect the applications where skin is used as conjugate information. Further, the algorithms that rely on visible spectrum images cannot be employed in the low/no light conditions especially during night times where the criticality of the application like human detection is higher. 
These problems which are  encountered in visible spectrum domain can be overcome by considering the images taken in the Near-infrared  (NIR) domain \cite{kong2005recent} or hyper spectral imaging \cite{pan2003face}. The information about the skin pixels is invariant of factors such as illumination conditions, ethnicity etc., in these domains. Moreover, most of the surveillance cameras that are used world-wide are NIR imaging devices. Thus, it is meaningful to pursue  the endeavour of detecting the skin pixels from the NIR images. 

\subsection{Problem setting and contributions}
The task of detection of skin pixels from an image is typically cast as a segmentation problem. Most of the classical approaches relied on the fact that the skin-pixels have a distinctive color pattern \cite{huynh2002skin,erdem2011combining} compared to other objects. In recent years, harnessing the power of Deep learning, skin segmentation problem has been dealt with using deep neural networks that show significant performance enhancement over the traditional methods \cite{long2015fully,ronneberger2015u,jones2002statistical}, albeit generalization across different illuminations still remains a challenge. 
While there exists sufficient literature on skin segmentation in the visible-spectrum, there is very little work done on segmenting the skin pixels in the NIR domain. Further, all the state-of-the-art Deep learning based segmentation algorithms demand large-scale annotated datasets to achieve good performance which is available in the case of visible-spectrum images but not the NIR images. Thus, building a fully-supervised skin segmentation network from scratch is not feasible for the NIR images because of the unavailability of the large-scale annotated data. However, the underlying concept of `skin-pixels' is the same across the images irrespective of the band in which they were captured. Additionally, the NIR and the Red-channel of the visible-spectrum are close in terms of their wavelengths.  Owing to these observations, we pose the following question in this paper - Can the labelled data (source) in the visible-spectrum (Red-channel) be used to perform skin segmentation in the NIR domain (target) \cite{pandey2020guided}? 
\par We cast the problem of skin segmentation from NIR images as a target-independent Unsupervised Domain Adaptation (UDA) task \cite{pandey2020target} where we consider the Red-channel of the visible-spectrum images as the source domain and NIR images as the target domain. The state-of-the-art UDA techniques demand access to the target data, albeit unlabelled, to adapt the source domain features to the target domain. In the present case, we do not assume existence of any data from the target domain, even unlabelled. This is an important desired attribute which ensures that a model trained on the Red-channel does not need any retraining with the data from NIR domain. The core idea is to sample the `nearest-clone' in the source domain to a given test image from the target domain. This is accomplished through a simultaneous sampling-cum-optimization procedure using a latent-variable deep neural generative network learned on the source distribution. Thus, given a target sample, its `nearest-clone' from the source domain is sampled and used as a proxy in the segmentation network trained only on the samples of the source domain. Since the segmentation network performs well on the source domain, it is expected to give the correct segmentation mask on the `nearest-clone' which is then assigned to the target image. Specifically, the core contributions of this work are listed as follows: 

\begin{enumerate}
    \item We cast the problem of skin segmentation from NIR images as a UDA segmentation task where we use the data from the Red-channel of the visible-range of the EM-spectrum to develop skin segmentation algorithm on NIR images.
    \item We propose a method for target-independent segmentation where the `nearest-clone' of a target image in the source domain is searched and used as a proxy in the segmentation network trained only on the source domain. 
    \item We theoretically prove the existence of the `nearest-clone' given that it can be sampled from the source domain with infinite data points. 
    \item We develop a joint-sampling and optimization algorithm using variational inference based generative model to search for the `nearest-clone' through implicit sampling in the source domain. 
    \item We demonstrate the efficacy of the proposed method for NIR skin segmentation over the state-of-the-art UDA segmentation methods on the two newly created skin segmentation datasets in NIR domain. The proposed method is also shown to reach SOTA performance on standard segmentation datasets like Synthia \cite{ros2016synthia} and Cityscapes \cite{cordts2016cityscapes}. 
\end{enumerate}
\section{Related Work}
In this section, we first review the existing methods for skin segmentation in the visible-range followed by a review of UDA methods for segmentation. 
\subsection{Skin Segmentation in Visible-range}
Methods for skin segmentation can be grouped into three categories, i.e. (i) Thresholding based methods \cite{kovac2003human, erdem2011combining, qiang2010robust}, (ii) Traditional machine learning techniques to learn a skin color model \cite{liu2010robust ,zaidan2014multi}, (iii) Deep learning based methods to learn an end-to-end model for skin segmentation \cite{al2013impact , wu2012skin , seow2003neural , chen2016skin , he2019semi}.  
The thresholding based methods focus  on  defining  a  specified  range  in  different color  representation spaces like (HSV)\cite{moallem2011novel} and orthogonal color space (YCbCr)\cite{hsu2002face,brancati2017human} to differentiate skin pixels. Traditional machine learning can be further divided into pixel based and region based methods. In pixel based methods, each pixel is classified as skin or non-skin without considering the neighbours \cite{taqa2010increasing} whereas region based approaches use spatial information to identify similar regions \cite{chen2007region}. 
In recent years, Fully convolutional neural networks (FCN) are employed to solve the problem \cite{long2015fully}.
\cite{ronneberger2015u} proposed a UNet architecture, consisting of an encoder-decoder structure with backbones like InceptionNet\cite{szegedy2016rethinking} and ResNet \cite{he2016deep}.
Holistic skin segmentation \cite{dourado2019domain} combine inductive transfer learning and UDA. They term this technique as cross domain pseudo-labelling and use it in an iterative manner  to  train and fine tune the model on the target domain. 
\cite{he2019semi} propose mutual guidance to improve skin detection with the usage of body masks as guidance. They use dual task neural network for joint detection with shared encoder and two decoders for detecting skin and body simultaneously. While all these methods offer different advantages, they do not generalize to low-light settings with NIR images, which we aim to solve through UDA. 
\subsection{Domain Adaptation for semantic segmentation}
Unsupervised Domain Adaptation aims to improve the performance  of deep neural networks on a target domain, using labels only from a source domain. UDA for segmentation task can be grouped into following  categories:
\subsubsection{Adversarial training based methods:}
These methods use the principles of adversarial learning \cite{hoffman2016fcns}, which generally consists of two networks. One predicts the segmentation mask of the input image coming from either source or target distribution while the other network acts as discriminator which tries to predict the domain of the images.
AdaptSegNet \cite{tsai2018learning} exploits structural similarity between the source and target domains in a multi-level adversarial network  framework. 
ADVENT \cite{vu2018advent} introduce entropy-based loss to directly penalize low-confident predictions on target domain. Adversarial training is used for structural adaptation of the target domain to the source domain.
CLAN \cite{luo2019taking}  considers category-level joint distribution and aligns each class with an adaptive adversarial loss. They reduce the weight of the adversarial loss for category-level aligned features while increasing the adversarial force for those that are poorly aligned. DADA \cite{vu2019dada} uses the geometry of the scene by simultaneously aligning the segmentation and depth-based information of source and target domains using adversarial training.
\subsubsection{Feature-transformation based methods:}
These methods are based on the idea of learning image-level or feature-level transformations between the source and the target domains. CyCADA \cite{Hoffman_cycada2017} adapts between domains using
both generative image space alignment and latent
representation space alignment. Image level adaptation is achieved with cycle loss, semantic consistency loss and pixel-level GAN loss while feature level adaptation employs feature-level GAN loss and task loss between true and predicted labels.
DISE \cite{chang2019all} aims to discover a domain-invariant structural feature by learning to  disentangle domain-invariant structural information of an image from its domain-specific texture information.
BDL \cite{li2019bidirectional}  involves  two  separated  modules a) image-to-image translation model b) segmentation adaptation  model, in two directions namely `translation-to-segmentation' and `segmentation-to-translation'. 
\section{Proposed method}
Most of the UDA methods assume access to the unlabelled target data which may not be available at all times. In this work, we propose a UDA segmentation technique by learning to find a data point from the source that is arbitrarily close (called the `nearest-clone') to a given target point so that it can used as a proxy in the segmentation network trained only on the source data. In the subsequent sections, we describe the methodology used to find the `nearest-clone' from the source distribution to a given target point. 
\subsection{Existence of nearest source point}
To start with, we show that for a given target data point, there exists a corresponding source data point,  that is arbitrarily close to, provided that infinite data points can be sampled from the source distribution. Mathematically, let $\mathcal{P}_s(\mathbf{x})$ denotes the source distribution and $\mathcal{P}_t(\mathbf{x})$ denotes any target distribution that is similar but not exactly same as $\mathcal{P}_s$ (Red-channel images are source and NIR images are target). Let the underlying random variable on which $\mathcal{P}_s$  and $\mathcal{P}_t$ are defined form a separable metric space $\{\mathscr{X,D}\}$ with $\mathscr{D}$ being some distance metric. Let $\mathcal{S}_n=\{\mathbf{x}_1, \mathbf{x}_2, \mathbf{x}_3,...., \mathbf{x}_n  \}$ be i.i.d points drawn from $\mathcal{P}_s(\mathbf{x})$ and $\tilde{\mathbf{x}}$ be a point from $\mathcal{P}_t(\mathbf{x})$. With this, the following lemma shows the existence of the `nearest-clone'. 
\begin{lemma}
If $\tilde{{\mathbf{x}}}_\mathcal{S} \in \mathcal{S}_n$ is the point such that $\mathscr{D\{\tilde{\mathbf{x}}},\tilde{{\mathbf{x}}}_\mathcal{S}\}< \mathscr{D\{\tilde{\mathbf{x}}},\mathbf{x} \} \ \forall \mathbf{x}\ \in\ \mathcal{S}_n $, as ${n\to\infty}$ (in $\mathcal{S}_n$), $\tilde{{\mathbf{x}}}_\mathcal{S}$ converges to  $\tilde{\mathbf{x}}$ with probability $1$. 
\end{lemma}
\begin{proof}
Let $\mathds{B}_r({\tilde{\mathbf{x}}}) = \{\mathbf{x}:\mathscr{D\{\tilde{\mathbf{x}}},\mathbf{x} \}\leq r\}$ be a closed ball of radius $r$ around $\tilde{\mathbf{x}}$ under the metric $\mathscr{D}$. Since $\mathscr{X}$ is a separable metric space \cite{cover1967nearest},  
\begin{equation}\mathbf{Prob}\big(\mathds{B}_r({\tilde{\mathbf{x}}})\big) \triangleq \int\limits_{\mathds{B}_r({\tilde{\mathbf{x}}})}\mathcal{P}_s(\mathbf{x}) \ d\mathbf{x} > 0, \forall r>0,
\end{equation}
With this, for any $\delta>0$, the probability that  none of the points in $\mathcal{S}_n$ are within the ball $\mathds{B}_\delta({\tilde{\mathbf{x}}})$ of radius $\delta$ is:
\begin{equation}
\mathbf{Prob}\bigg[ \min \limits_{i=1,2..,n} \mathscr{D\{\mathbf{x}_\textit{i},\tilde{\mathbf{x}}}\} \geq \delta  \bigg] = \big[ 1- \mathbf{Prob}\big(\mathds{B}_\delta({\tilde{\mathbf{x}}})\big) \big]^n
\end{equation}Therefore, the probability of $\tilde{{\mathbf{x}}}_\mathcal{S}$ (the closest point to $\tilde{\mathbf{x}}$) lying within $\mathds{B}_\delta({\tilde{\mathbf{x}}})$ is: 
\begin{align}
\mathbf{Prob}\bigg[ \tilde{{\mathbf{x}}}_\mathcal{S} \in  \mathds{B}_\delta({\tilde{\mathbf{x}}})  \bigg] &= 1 -\big[ 1- \mathbf{Prob}\big(\mathds{B}_\delta({\tilde{\mathbf{x}}})\big) \big]^n\\
&= 1\ \ as \ n\rightarrow \infty
\end{align}
Thus, given any infinitesimal $\delta>0$, with probability $1$, $\exists\ \tilde{{\mathbf{x}}}_\mathcal{S} \in \mathcal{S}_n$ (`nearest-clone') that is within  $\delta$ distance from $\tilde{\mathbf{x}}$ as $n\rightarrow \infty$\qed
 \end{proof}
\begin{figure*}[h]
\centering
    \includegraphics[width=1.0\textwidth,height=.4\textwidth]{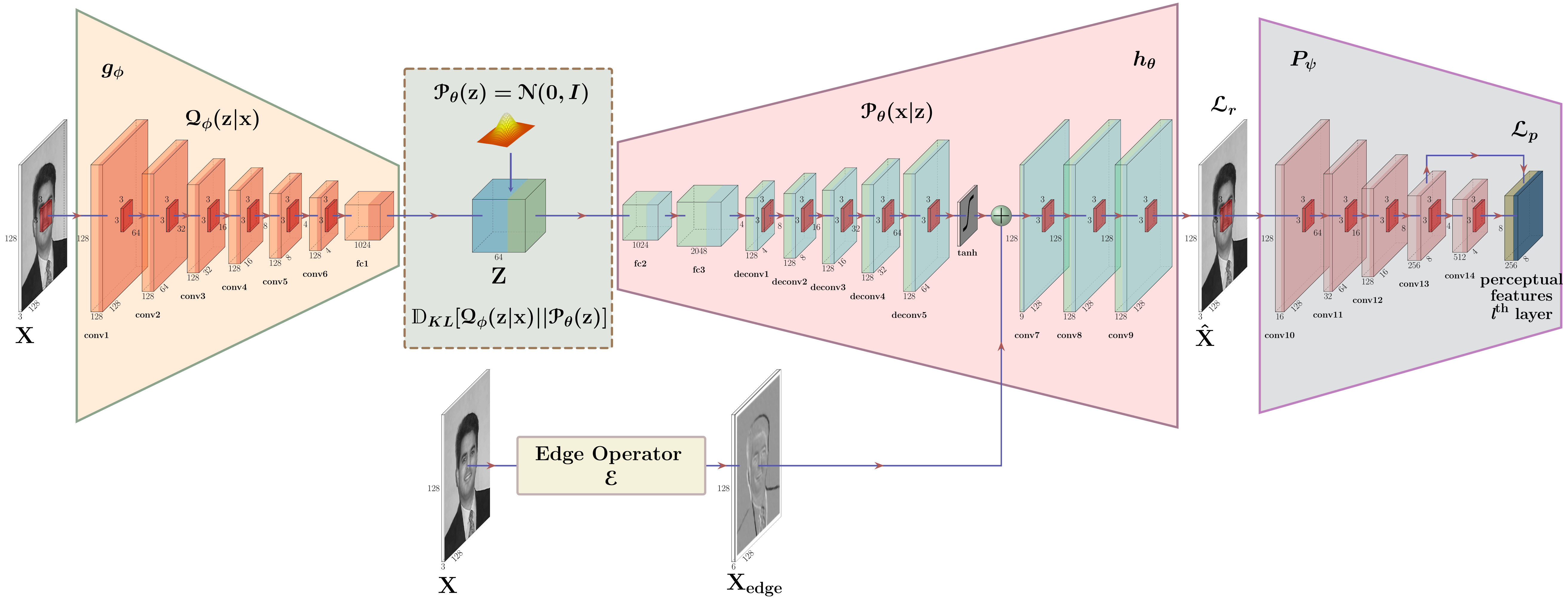}
    \caption{VAE training. Edges of an input image are concatenated with the features from the decoder $h_{\theta}$. Encoder and decoder parameters $\phi$, $\theta$ are optimized with reconstruction loss $\mathcal{L}_{r}$, KL-divergence loss $\mathds{D}_{KL}$ and perceptual loss $\mathcal{L}_{p}$. Perceptual model $P_{\psi}$ is trained on source samples. A zero mean and unit variance isotropic Gaussian prior is imposed on the latent space $\bm{\textbf{z}}$. }
    \label{fig:architecture}
\end{figure*} 
While Lemma 1 guarantees the existence of a `nearest-clone', it demands the following two conditions: 
\begin{itemize}
\item It should be possible to sample infinitely from the source distribution $\mathcal{P}_s$.
\item It should be possible to search for the `nearest-clone' in the $\mathcal{P}_s$, for a target sample $\tilde{\mathbf{x}}$ under the distance metric $\mathscr{D}$.
\end{itemize}
We propose to employ Variational Auto-encoding based sampling models on the source distribution to simultaneously sample and find the `nearest-clone' through an optimization over the latent space.  
\subsection{Variational Auto-Encoder for source sampling}
Variational Auto-Encoders (VAEs) \cite{kingma2013auto} are a class of latent-variable generative models that are based on the principles of variational inference where the variational distribution, $\mathcal{Q}_\phi(\mathbf{z|x})$ is used to approximate the intractable true posterior $\mathcal{P}_\theta(\mathbf{z|x})$. The log-likelihood of the observed data is decomposed into two terms, an irreducible non-negative KL-divergence between  $\mathcal{P}_\theta(\mathbf{z|x})$ and  $\mathcal{Q}_\phi(\mathbf{z|x})$ and the Evidence Lower Bound (ELBO) term which is given by Eq. \ref{elbo1}.
\begin{figure}[h]
\begin{center}
    \includegraphics[width=0.69\textwidth,height=.42\textwidth]{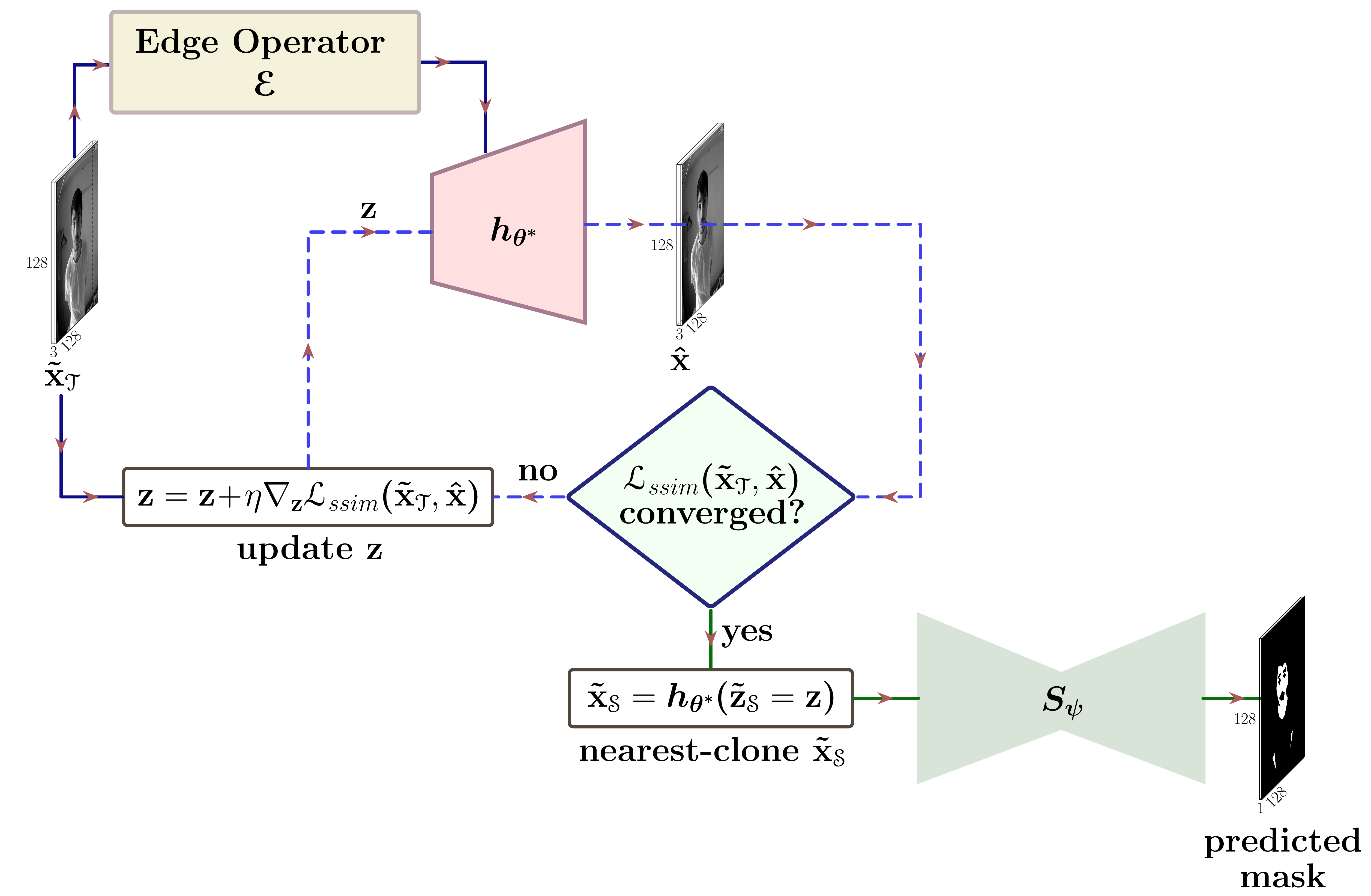}
\caption{Latent Search procedure during inference with GLSS. The latent vector $\mathbf{z}$ is initialized with a random sample drawn from $\mathcal{N}(0,1)$. Iterations over the latent space $\rvz$ are performed to minimize the $\mathcal{L}_{ssim}$ loss between the input target image $\bm{\tilde{\mathbf{x}}}_{\mathcal{T}}$ and the predicted target image $\mathbf{\hat x}$ (blue dotted lines). After convergence of $\mathcal{L}_{ssim}$ loss, the optimal latent vector $\bm{\tilde{\mathbf{z}}}_{\mathcal{S}}$, generates the closest clone $ \bm{\tilde{\mathbf{x}}}_{\mathcal{S}}$ which is used to predict the mask of $\bm{\tilde{\mathbf{x}}}_{\mathcal{T}}$ using the segmentation network $S_\psi$.}
\label{fig:LatentSearch}
\end{center}
\end{figure}
\begin{equation}
 \label{elbo1}
\ln\mathcal{P}_{\theta}(\mathbf{x})=\mathcal{L(\theta,\phi)}+\mathds{D}_{KL}[\mathcal{Q}_{\phi}(\mathbf{z}|\mathbf{x})||\mathcal{P}_{\theta}(\mathbf{z|x})]
\end{equation} where, 
\begin{equation}
\mathcal{L(\theta,\phi)=\mathbb{E}}_{\mathcal{Q}_{\phi}(\mathbf{z}|\mathbf{x})}[\ln\left(\mathcal{P}_{\theta}(\mathbf{x|}\mathbf{z})\right)]-\mathds{D}_{KL}[\mathcal{Q}_{\phi}(\mathbf{z}|\mathbf{x})||\mathcal{P}_{\theta}(\mathbf{z})]
    \label{elbo}
\end{equation}
The non-negative KL-term in Eq. \ref{elbo1} is irreducible and thus,  $\mathcal{L(\theta,\phi)}$ serves as a lower bound on the data log-likelihood which is maximized in a VAE by parameterizing $\mathcal{Q}_{\phi}(\mathbf{z|x})$ and $\mathcal{P}_{\phi}(\mathbf{x|z})$ using probabilistic encoder $g_\phi$ (that outputs the parameters $\mu_\rvz$ and $\sigma_\rvz$ of a distribution) and decoder $h_\theta$ neural networks. The latent prior $\mathcal{P}_{\theta}(\mathbf{z})$ is taken to be arbitrary prior on $\mathbf{z}$ which is usually a $0$ mean and unit variance Gaussian distribution. After training, the decoder network is used as a sampler for $\mathcal{P}_s(\mathbf{x})$ in a two-step process: (i) Sample $\mathbf{z} \sim \mathcal{N}(0,I)$, (ii) Sample $\mathbf{x}$ from $\mathcal{P}_\theta(\mathbf{x|z})$.
The likelihood term in Eq. \ref{elbo1} is approximated using norm based losses and it is known to result in blurry images. Therefore, we use the perceptual loss \cite{johnson2016perceptual} along with the standard norm based losses. Further, since the edges in images are generally invariant across the source and target domains, we extract edge of the input image and use it in the decoder of the VAE via a skip connection, as shown in Fig. \ref{fig:architecture}. This is shown to reduce the blur in the generated images. Fig. \ref{fig:architecture} depicts the entire VAE architecture used for training on the source data. 

\subsection{VAE Latent Search for finding the `nearest-clone'}
As described, the objective of the current work is to search for the nearest point in the source distribution, given a sample from target distribution. The decoder $h_\theta$ of a VAE trained on the source distribution $\mathcal{P}_s(\mathbf{x})$, outputs a new sample using the Normally distributed latent sample as input. That is, 
\begin{equation}
\forall \mathbf{z}\sim \mathcal{N}(0,I), \hat{\mathbf{x}}=h_\theta(\mathbf{z})\sim \mathcal{P}_s(\hat{\mathbf{x}})
\end{equation}

With this, our goal is to find the `nearest-clone' to a given target sample. That is, given a $\tilde{\mathbf{x}} \sim \mathcal{P}_t(\mathbf{x})$, find $\tilde{{\mathbf{x}}}_\mathcal{S}$ as follows:

\begin{equation}
\tilde{{\mathbf{x}}}_\mathcal{S}=h_\theta(\tilde{{\mathbf{z}}}_\mathcal{S}):\bigg \{ \mathscr{D\{\tilde{\mathbf{x}}},\tilde{{\mathbf{x}}}_\mathcal{S}\}< \mathscr{D\{\mathbf{x}},\tilde{\mathbf{x}} \} \ \forall {\mathbf{x}}\ =  h_\theta(\mathbf{z}) \sim \mathcal{P}_s({\mathbf{x}})
\label{obj}
\end{equation}

Since $\mathscr{D}$ is pre-defined and $h_\theta(\mathbf{z})$ is a deep neural network, finding $\tilde{{\mathbf{x}}}_\mathcal{S}$ can be cast as an optimization problem over $\mathbf{z}$ with minimization of $\mathscr{D}$ as the objective. Mathematically, 

\begin{equation}
\tilde{{\mathbf{z}}}_\mathcal{S}= \underset{\mathbf{z}}{\arg\min}\ \mathscr{D}\big ( \tilde{\mathbf{x}}, h_\theta(\mathbf{z}) \big )
\label{obj}
\end{equation}

\begin{equation}
\tilde{{\mathbf{x}}}_\mathcal{S}=  h_\theta(\tilde{{\mathbf{z}}}_\mathcal{S}) 
\end{equation}
The optimization problem is Eq. \ref{obj} can be solved using gradient-descent based techniques on the decoder network $h_{\theta^{\ast}}$ $\big ( \theta^{\ast}$ are the parameters of the decoder network trained only on the source samples $\mathcal{S}_n \big )$  with respect to $\mathbf{z}$. This implies that given any input target image, the optimization problem in Eq. \ref{obj} will be solved to find its `nearest-clone' in the source distribution which is used as a proxy in the segmentation network trained only on $\mathcal{S}_n$. We call the iterative procedure of finding $\tilde{{\mathbf{x}}}_\mathcal{S}$  through optimization using $h_{\theta^{\ast}}$ as the Latent Search (LS). Finally, inspired by the observations made in \cite{hore2010image}, we propose to use structural similarity index (SSIM) \cite{wang2004image} based loss $\mathcal{L}_{ssim}$ for $\mathscr{D}$ to conduct the Latent Search. Unlike norm based losses, SSIM loss helps in preservation of structural information which is needed for segmentation. Fig. \ref{fig:LatentSearch} depicts the complete inference procedure employed in the proposed method named as the Generative Latent Search for Segmentation (GLSS).
\section{Implementation Details}
\subsection{Training}
Architectural details of the VAE used are shown in Fig. \ref{fig:architecture}. Sobel operator is used to extract the edge information of the input image which is concatenated with one of the layers of the Decoder via a \textit{tanh} non linearity as shown in Fig. \ref{fig:architecture}. The VAE is trained using (i) the Mean squared error reconstruction loss $\mathcal{L}_{r}$ and KL divergence $\mathds{D}_{KL}$ and (ii) the perceptual loss $\mathcal{L}_{p}$ for which the features are extracted from the $l^{\text{th}}$ layer (a hyper-parameter) of the DeepLabv3+ \cite{deeplabv3plus2018} (Xception backbone \cite{chollet2017xception}) and the UNet \cite{ronneberger2015u} (EfficientNet backbone \cite{tan2019efficientnet}) segmentation networks. The segmentation network ($S_{\psi}$ in Fig. \ref{fig:LatentSearch}) is either DeepLabv3+ or UNet  and is trained on the source dataset. For traning $S_{\psi}$, we use combination of binary cross-entropy ($\mathcal{L}_{bce}$) and dice coefficient loss ($\mathcal{L}_{dise}$) for UNet with RMSProp (lr = 0.001) as optimizer and  binary focal loss ($\mathcal{L}_{focal}$) \cite{lin2017focal} with $\gamma$ = 2.0, $\alpha$ = 0.75 and RMSProp (lr=0.01) as optimizer for DeepLabv3+. 
For the VAE , the hidden layers of Encoder and Decoder networks use Leaky ReLU and \textit{tanh} as activation functions with the dimensionality of the latent space being 64. VAE is trained using standard gradient descent procedure with RMSprop ($\alpha$=0.0001) as optimizer. We train VAE for 100 to 150 epochs with batchsize 64. 
\subsection{Inference}
Once the VAE is trained on the source dataset, given an image $\tilde{\mathbf{x}}_\mathcal{T}$ from the target distribution,
the Latent Search algorithm searches for an optimal latent vector $\tilde{\mathbf{z}}_\mathcal{S}$ that generates its `nearest-clone' $\tilde{\mathbf{x}}_\mathcal{S}$ from $\mathcal{P}_S$. The search is performed by minimizing the SSIM loss $ \mathcal{L}_{ssim}$ between the input target image $\tilde{\mathbf{x}}_\mathcal{T}$ and the VAE-reconstructed target image, using a gradient-descent based optimization procedure such as  ADAM \cite{kingma2014adam} with $\alpha=0.1$, $\beta_1=0.9$ and $\beta_2= 0.99$. The Latent Search is performed for $K$ (hyper-parameter) iterations over the latent space of the source for a given target image. Finally, the segmentation mask for the input target sample is assigned the same as the one given by the segmentation network $S_{\psi}$, which is trained on source data, on the `nearest-clone' $\tilde{\mathbf{x}}_\mathcal{S}$. 
Latent Search for one sample takes roughly 450 ms and 120 ms on SNV and Hand Gesture datasets respectively. Please refer supplementary material for more details.
\section{Experiment and Results}
\subsection{Datasets}
We consider the Red-channel of the COMPAQ dataset \cite{jones2002statistical} as our source data. It consists of 4675 RGB images with the corresponding annotations of the skin. 
Since there is no publicly available dataset with NIR images and corresponding skin segmentation labels, we create and use two NIR datasets (publicly available) as targets. The first one named as the Skin NIR Vision (SNV), consists of 800 images of multiple human subjects taken in different scenes, captured using a WANSVIEW 720P camera in the night-vision mode. The captured images cover wide range of scenarios for skin detection task like presence of multiple humans, backgrounds similar to skin color, different illuminations, saturation levels and different postures of subjects to ensure diversity.
Additionally, we made use of the publicly available multi-modal Hand Gesture dataset\footnote{\url{https://www.gti.ssr.upm.es/data/MultiModalHandGesture_dataset}} as another target dataset which we call as Hand Gesture dataset. This dataset covers 16 different hand-poses of multiple subjects. We randomly sampled  500 images in order to cover illumination changes and diversity in hand poses.
Both SNV and Hand Gesture datasets are manually annotated with precision. 
\subsection{Benchmarking on SNV and Hand Gesture datasets}
To begin with, we performed supervised segmentation experiments on both SNV and Hand Gesture datasets with 80-20 train-test split using SOTA segmentation algorithms.
\bgroup
\setlength{\tabcolsep}{18pt}
\begin{table}[h]
\caption{Benchmarking Skin NIR Vision (SNV) dataset and Hand Gesture dataset on standard segmentation architectures with 80-20 train-test split.}
\begin{center}
\scalebox{0.77}{
\begin{tabular}{lcccc}
    \toprule
        \multicolumn{1}{c}{} & \multicolumn{2}{c}{SNV} & \multicolumn{2}{c}{Hand Gesture}\\
        Method&IoU&Dice&IoU&Dice
        \\
    \midrule
    FPN \cite{lin2017feature} &0.792&0.895&0.902&0.950\\
    UNet \cite{ronneberger2015u}&0.798&0.890&0.903&0.950\\
    DeepLabv3+ \cite{deeplabv3plus2018} & 0.750 & 0.850 & 0.860 & 0.924\\
    Linknet \cite{chaurasia2017linknet} & 0.768 & 0.872 & 0.907 & 0.952\\
    PSPNet \cite{zhao2017pyramid} & 0.757 & 0.850 & 0.905 & 0.949\\
    \bottomrule
  \end{tabular}
  }
\end{center}
\label{Table:Datasetbench}
\end{table}
\egroup

Table \ref{Table:Datasetbench}  shows the standard performance metrics such as IoU and Dice-coefficient calculated using FPN \cite{lin2017feature}, UNet \cite{ronneberger2015u}, LinkNet \cite{chaurasia2017linknet}, PSPNet \cite{zhao2017pyramid}, all with EfficientNet \cite{tan2019efficientnet} as backbone and DeepLabv3+ \cite{deeplabv3plus2018} with Xception network \cite{chollet2017xception} as backbone.
It is seen that SNV dataset (IoU $\approx$ 0.79) is slightly complex as compared to Hand Gesture dataset (IoU $\approx$ 0.90).
\bgroup
\setlength{\tabcolsep}{4.4pt}
\begin{table}[h]
\caption{Empirical analysis of GLSS along with standard UDA methods. IoU and Dice-coefficient are computed for both SNV and Hand Gesture datasets using UNet  and DeepLabv3+ as segmentation networks.}
\begin{center}
\scalebox{0.78}{
\begin{tabular}{l|cccc|cccc}
    \toprule
    \multicolumn{1}{c|}{} & \multicolumn{4}{c|}{SNV} & \multicolumn{4}{c}{Hand Gesture}\\
     & \multicolumn{2}{c|}{UNet} & \multicolumn{2}{c|}{DeepLabv3+} & \multicolumn{2}{c|}{UNet} &\multicolumn{2}{c}{DeepLabv3+}\\
     Models &\multicolumn{1}{c|}{IoU} & \multicolumn{1}{c|}{Dice} & \multicolumn{1}{c|}{IoU} &\multicolumn{1}{c|}{Dice} & \multicolumn{1}{c|}{IoU} & \multicolumn{1}{c|}{Dice} & \multicolumn{1}{c|}{IoU} &\multicolumn{1}{c}{Dice}\\
    \midrule
    Source Only & 0.295 & 0.426 & 0.215 & 0.426 & 0.601 & 0.711 & 0.505 & 0.680\\
    AdaptSegnet \cite{tsai2018learning} & 0.315 & 0.435 & 0.230 & 0.435 &0.641&0.716& 0.542 & 0.736\\
    Advent \cite{vu2018advent} & 0.341 & 0.571 & 0.332 &0.540 &0.612 & 0.729 & 0.508 & 0.689\\
    CLAN \cite{luo2019taking} & 0.248& 0.442 & 0.225 &0.426 &0.625 & 0.732 & 0.513 & 0.692\\
    BDL \cite{li2019bidirectional} & 0.320 & 0.518 & 0.301 &0.509 &0.647 & 0.720 & 0.536 & 0.750\\
    DISE \cite{chang2019all} & 0.341  & 0.557  & 0.339  &0.532  &0.672  & 0.789  & 0.563 & 0.769 \\
    DADA \cite{vu2019dada} & 0.332  & 0.534  & 0.314  &0.521  &0.643  & 0.743  &0.559  & 0.761 \\
    ours (GLSS) & \textbf{0.406}  & \textbf{0.597}  & \textbf{0.385}  & \textbf{0.597}  & \textbf{0.736}  & \textbf{0.844}  & \textbf{0.698}  & \textbf{0.824} \\
    \bottomrule
\end{tabular}
}
\end{center}
\label{Table:2}
\end{table} 
\egroup
\subsection{Baseline UDA Experiments}
\subsubsection{SNV and Hand Gesture dataset:}
We have performed the UDA experiments with the SOTA UDA algorithms using Red-channel of the COMPAQ Dataset  \cite{jones2002statistical} as the source and SNV and Hand Gesture as the target.
Table \ref{Table:2} compares the performance of proposed GLSS algorithm with six SOTA baselines along with the Source Only case (without any UDA). We have used entire target dataset for IoU and Dice-coefficient evaluation. 
\begin{figure}[h]
\centering
    \includegraphics[width=1.0\textwidth,height=0.672\textwidth]{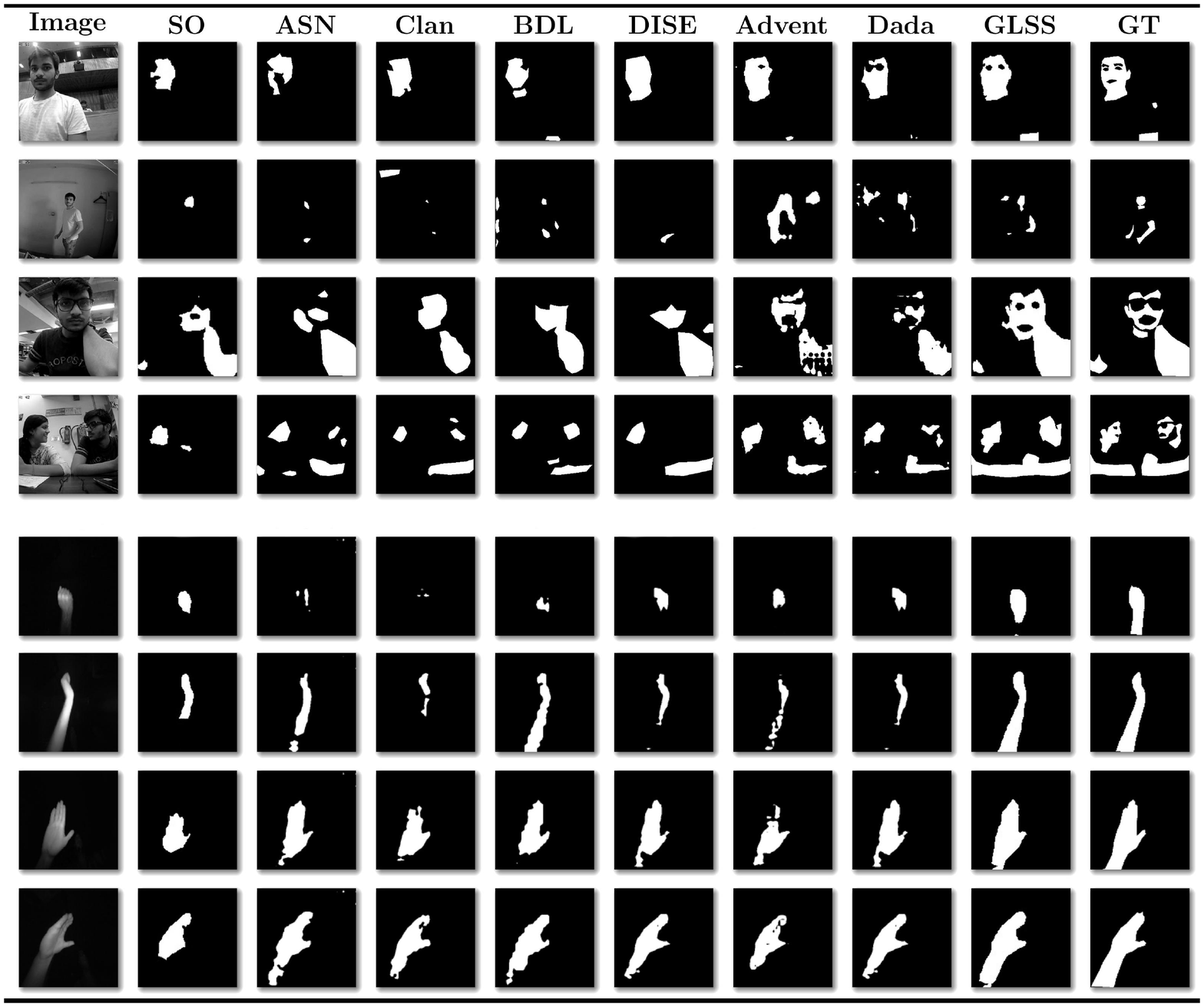}
    \caption{Qualitative comparison of predicted segmentation skin masks on SNV and Hand Gesture datasets with standard UDA methods. Top four rows shows skin masks for SNV dataset and the last four are the masks for Hand Gesture dataset. It is evident that GLSS predicted masks are very close to the GT masks as compared to other UDA methods. (SO=Source Only, ASN=AdaptSegNet \cite{tsai2018learning}, GT=Ground Truth).
    }
    \label{fig:gridmasks}
\end{figure}
 Two architectures, DeepLabv3+ and UNet, are employed for the segmentation network ($S_\psi$). It can be seen that although all the UDA SOTA methods improve upon the Source Only performance, GLSS offers significantly better performance despite not using any data from the target distribution. Hence, it may be empirically inferred that GLSS is successful in producing the `nearest-clone' through implicit sampling from the source distribution and thereby reducing the domain shift. It is also observed that the performance of the segmentation network $S_\psi$ does not degrade on the source data with GLSS. 
The predicted masks with DeepLabv3+ are shown in Fig. \ref{fig:gridmasks} for SNV and Hand Gesture datasets, respectively. It can be seen that GLSS is able to capture fine facial details like eyes, lips and body parts like hands, better as compared to SOTA methods. It is also seen that the predicted masks for Hand Gesture dataset are sharper in comparison to other methods. 
Most of the methods work with the assumption of spatial and structural similarity between the source and target data. Since our source and target datasets do not have similar backgrounds, the methods that make such assumptions perform poorer on our datasets. We observed that for methods like BDL, the image translation between NIR images and Red channel images is not effective for skin segmentation task.

\subsubsection{Standard UDA task:}
We use standard UDA methods along with GLSS on standard domain adaptation datasets such as Synthia \cite{ros2016synthia} and Cityscapes \cite{cordts2016cityscapes}. As observed from Table \ref{tab:synthiatocityscape}, even with large domain shift, GLSS finds a clone for every target image that is sampled from the source distribution while preserving the structure of the target image.
\begin{table}[h]
\caption{Empirical analysis of GLSS on standard domain adaptaion task of adapting  Synthia \cite{ros2016synthia} to Cityscapes \cite{cordts2016cityscapes}. We calculate the mean IoU for 13 classes (mIoU) and 16 classes (mIoU*).}
\begin{center}
\scalebox{0.775}{
\begin{tabular}{c|c|c}
    \toprule
    Models  &  mIoU & mIoU*   \\
    \midrule
    AdaptsegNet \cite{tsai2018learning} &46.7& -\\
    Advent \cite{vu2018advent} &48.0&41.2\\
    BDL \cite{li2019bidirectional} &51.4&-\\
    CLAN \cite{luo2019taking} &47.8&-\\
    DISE \cite{chang2019all} &48.8&41.5\\
    DADA \cite{vu2019dada} &49.8&42.6\\
    ours(GLSS) &\textbf{52.3}&\textbf{44.5}\\
\bottomrule
\end{tabular}
}
\end{center}
\label{tab:synthiatocityscape}
\end{table}
\subsection{Ablation Study}
We have conducted several ablation experiments on GLSS using both SNV and Hand Gesture datasets using DeepLabv3+ as segmentation networks ($S_\psi$) to ascertain the utility of different design choices we have made in our method. 
\begin{figure}[h]
\begin{center}
    \includegraphics[width=0.75\textwidth,height=.37\textwidth]{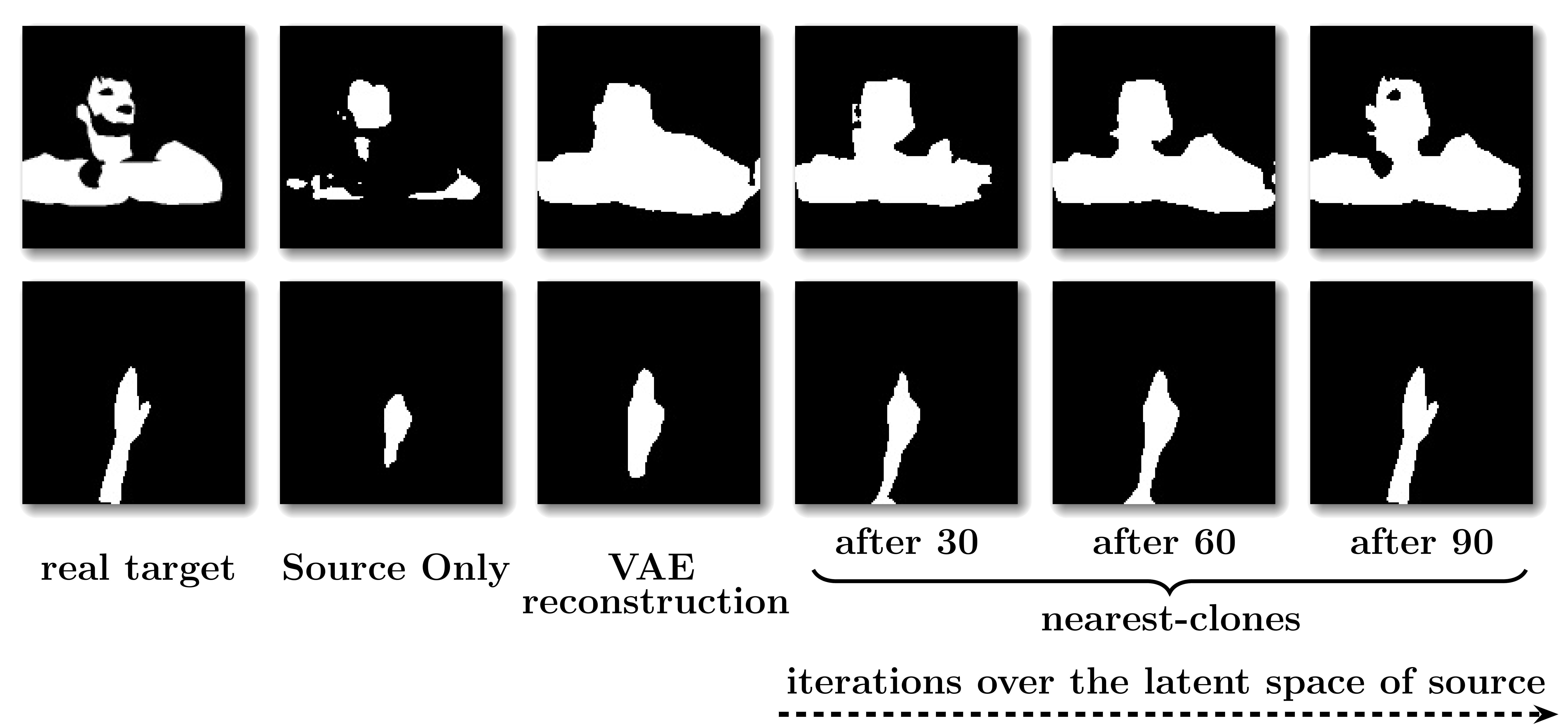}
    \caption{Illustration of Latent Search in GLSS. Real target is a ground truth mask. Source Only masks are obtained from target samples by training segmentation network $S_{\psi}$ on source dataset. Prior to the LS, skin masks are obtained from VAE reconstructed target samples. It is evident that predicted skin masks improve as the LS progresses. The predicted masks for the `nearest-clones' are shown after every 30 iterations.}
    \label{fig:LStransformMask}
\end{center}
\end{figure}
\begin{figure}[h]

      \subfloat[SNV]{\includegraphics[width=0.48\linewidth,height=0.29\linewidth]{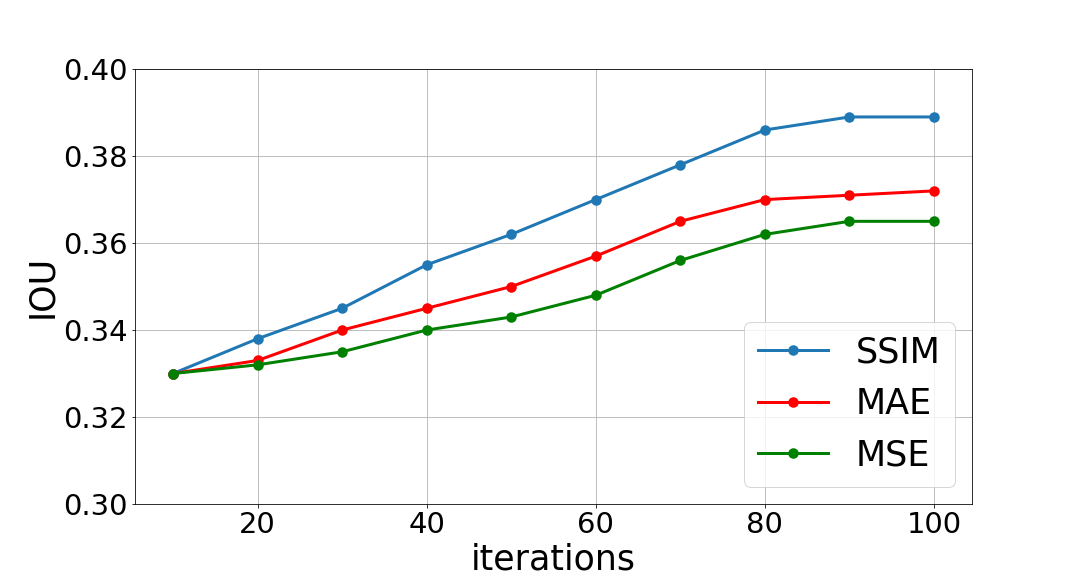}
        }
    \hfill
      \subfloat[Hand Gesture]{\includegraphics[width=0.48\linewidth,height=0.29\linewidth]{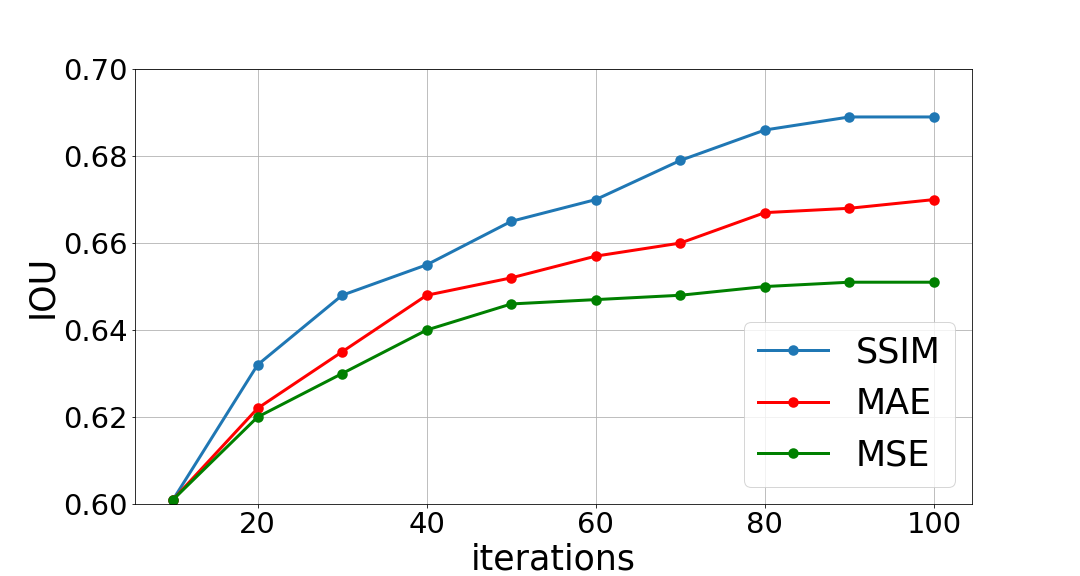}
      }
    \caption{Performance of gradient-based Latent Search during inference on target SNV and Hand Gesture images using different objective functions; MSE, MAE, SSIM loss. DeepLabv3+  is employed as segmentation network. It is evident that the losses saturate at around 90-100 iterations.}
    \label{fig:losa}
\end{figure}
\subsubsection{Effect of number of iterations on LS:}
The inference of GLSS involves a gradient-based optimization through the decoder network $h_{\theta^{\ast}}$ to generate the `nearest-clone' for a given target image. In Fig. \ref{fig:LStransformMask}, we show the skin masks of the transformed target images after every 30 iterations. 
It is seen that with the increasing number of iterations, the predicted skin masks improves using GLSS as the `nearest-clones' are optimized during the Latent Search procedure. 
We plot the IoU as a function of the number of iterations during Latent Search as shown in Fig. \ref{fig:losa} where it is seen that it saturates around 90-100 iterations that are used for all the UDA experiments described in the previous section. 
\subsubsection{Effect of Edge concatenation:}
As discussed earlier, edges extracted using Sobel filter on input images are concatenated with one of the layers of decoder for both training and inference.  It is seen from Table \ref{tab:ablation} that IoU improves for both the target datasets with concatenation of edges.
\setlength{\tabcolsep}{12.4pt}
\begin{table}[h]
\caption{Ablation of different components of GLSS during training and inference; Edge, perceptual loss $\mathcal{L}_{p}$ and Latent Search (LS).}
\begin{center}
\scalebox{0.77}{
\begin{tabular}{ccc|c|c}
    \toprule
    Edge & $\mathcal{L}_{p}$ & LS & SNV IoU & Hand Gesture IoU   \\
    \midrule
    &&&0.112& 0.227\\
    \checkmark&&&0.178&0.560\\
    &\checkmark&&0.120&0.250\\
    &&\checkmark&0.128&0.238\\
    \checkmark&\checkmark&&0.330&0.615\\
    &\checkmark&\checkmark&0.182&0.300\\
    \checkmark&&\checkmark&0.223&0.58\\
    \checkmark&\checkmark&\checkmark&0.385&0.698\\
\bottomrule
\end{tabular}
}
\end{center}
\label{tab:ablation}
\end{table} 
It is observed that without the edge concatenation, the generated images (`nearest-clones') are blurry thus the segmentation network fails to predict sharper skin masks.
\subsubsection{Effect of Perceptual loss $\mathcal{L}_{p}$:}
We have introduced a perceptual model $P_{\psi}$, trained on source samples. It ensures that the VAE reconstructed image is semantically similar to the input image unlike the norm based losses. Table \ref{tab:ablation} clearly demonstrates the improvement offered by the use of perceptual loss while training the VAE.
\subsubsection{Effect of SSIM for Latent Search:}
Finally, to validate the effect of SSIM loss for Latent Search, we plot the IoU metric using two norm based losses MSE (Mean squared error) and MAE (Mean absolute error) for the Latent Search procedure as shown in Fig. \ref{fig:losa}. On both the datasets, it is seen  that SSIM is consistently better than the norm based losses at all iterations affirming the superiority of the SSIM loss in preserving the structures while finding the `nearest-clone'.  

\section{Conclusion}
 In this paper, we addressed the problem of skin segmentation from NIR images. Owing to the non-existence of large-scale labelled NIR datasets for skin segmentation, the problem is casted as Unsupervised Domain Adaptation where we use the segmentation network trained on the Red-channel images from a large-scale labelled visible-spectrum dataset for UDA on NIR data. We propose a novel method for UDA without the need for the access to the target data (even unlabelled). Given a target image, we sample an image from the source distribution that is `closest' to  it under a distance metric. We show that such a `closest' sample exists and describe a procedure using an optimization algorithm over the latent space of a VAE trained on the source data. We demonstrate the utility of the proposed method along with the comparisons with SOTA UDA segmentation methods on the skin segmentation task on two NIR datasets that were created. Also, we reach SOTA performance on Synthia and Cityscapes datasets for semantic segmentation of urban scenes.
 \clearpage
\bibliographystyle{splncs04}
\bibliography{eccv2020submission}

\makeatletter
\newcounter{phase}[algorithm]
\newlength{\phaserulewidth}
\newcommand{\setphaserulewidth}{\setlength{\phaserulewidth}}
\newcommand{\phase}[1]{%
  \vspace{-2.2ex}
  \Statex\leavevmode\llap{\rule{\dimexpr\labelwidth+\labelsep}{\phaserulewidth}}\rule{\linewidth}{\phaserulewidth}
  \Statex\strut\refstepcounter{phase}\textbf{#1}
  \vspace{-2.2ex}\Statex\leavevmode\llap{\rule{\dimexpr\labelwidth+\labelsep}{\phaserulewidth}}\rule{\linewidth}{\phaserulewidth}}
\makeatother
\setphaserulewidth{.35pt}

\title{Unsupervised Domain Adaptation for Semantic Segmentation of NIR Images through Generative Latent Search
\\
$-$Supplementary$-$} 


\titlerunning{UDA for Semantic Segmentation of NIR Images through GLS}
%
\author{Prashant Pandey\thanks{equal contribution}\orcidID{0000-0002-6594-9685} \and
Aayush Kumar Tyagi\index{Tyagi, Aayush Kumar}\printfnsymbol{1}\orcidID{0000-0002-3615-7283} \and
Sameer Ambekar\orcidID{0000-0002-8650-3180} \and \\ Prathosh AP\orcidID{0000-0002-8699-5760}}
%
\authorrunning{P. Pandey et al.}
%
\institute{Indian Institute of Technology Delhi \\
\email{getprashant57@gmail.com, aayush16081@iiitd.ac.in, ambekarsameer@gmail.com, prathoshap@iitd.ac.in}}

\maketitle
\section{Datasets}
\begin{figure}[h]
\begin{center}
    \includegraphics[width=0.6\textwidth,height=.29\textwidth]{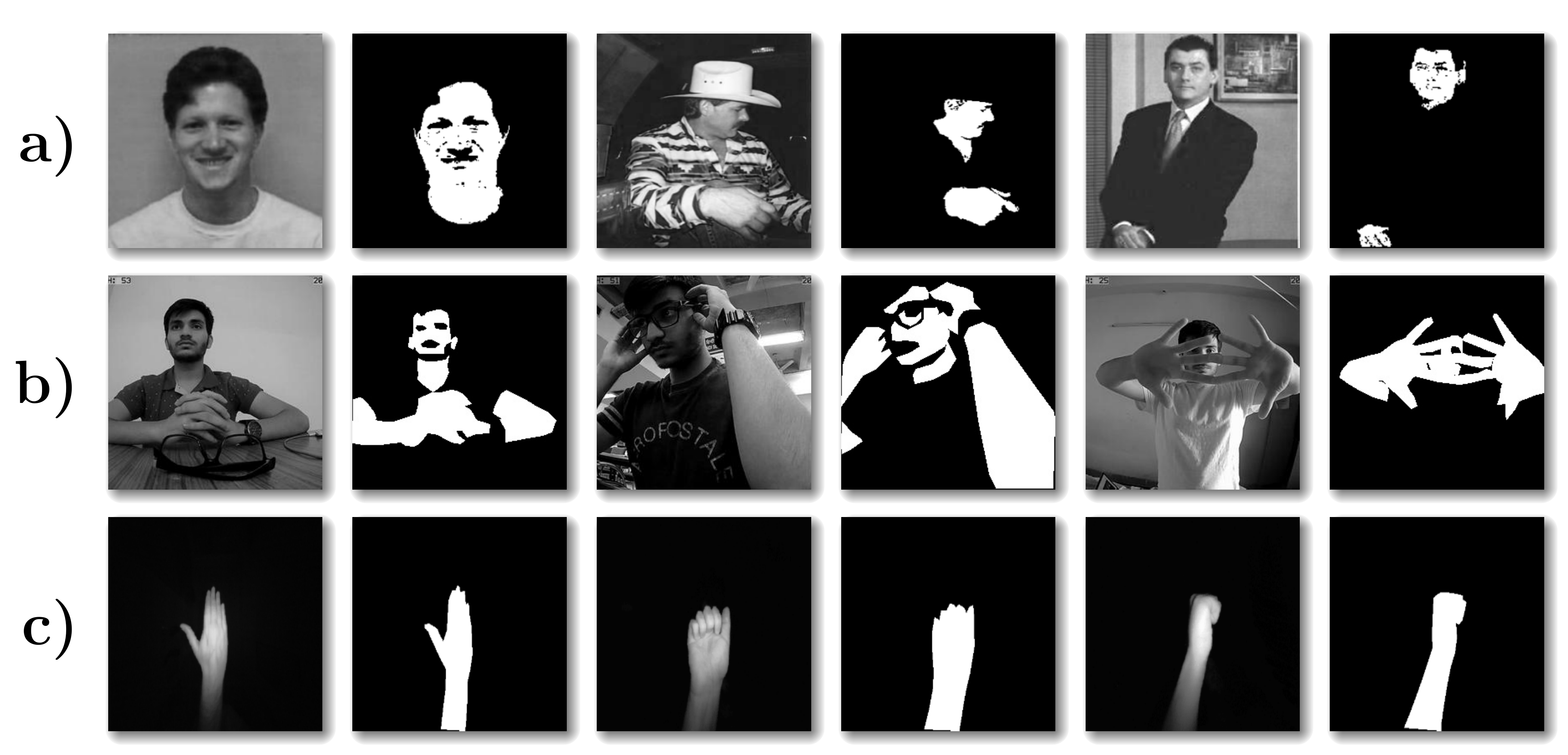}
    \caption{a) shows samples of COMPAQ dataset \cite{jones2002statistical} images with only Red-channel present b) contains samples from SNV dataset c) contains samples from Hand Gesture dataset.
    }
    \label{fig:DatasetExample}
\end{center}
\end{figure}
Each row of Fig. \ref{fig:DatasetExample} shows few images with the corresponding skin-mask pairs from  COMPAQ, SNV and Hand Gesture datasets respectively.
\begin{algorithm}[bth!]
    \caption{\textbf{Generative Latent Search for Segmentation (GLSS)}}
    \label{alg:trainalgo}
    
    \hspace*{\algorithmicindent} 
    
    \begin{algorithmic}[1] 
    \vspace*{.1cm}
    \phase{{Training VAE on source samples}}
    \textbf{Input}: Source dataset $\mathcal{S}_{n}= \{\rvx_{1},..., \rvx_{n} \}$, Number of source samples $n$, Encoder $g_{\phi}$, Decoder $h_{\theta}$, Trained Perceptual Model ${P}_{\psi}$, Learning rate $\eta$, Batchsize $B$. \textbf{Output}: Optimal parameters $\phi^*$, $\theta^*$.
    \State Initialize parameters $\phi$, $\theta$
    \REPEAT
    \State sample batch $\{\rvx_{i}\}$ from dataset $\mathcal{S}_{n},$ for $ i = 1,..., B$  
    \State $\mu_\rvz^{(i)}, \sigma_\rvz^{(i)} \gets g_{\phi} ( \rvx_{i})$
    \State sample ${\rvz}_{i} \sim \mathcal{N}(\mu_\rvz^{(i)},$ ${\sigma_\rvz^{(i)}}^2)$
    \vspace*{0.02cm}
    \State $ \mathcal{L}_{r} \gets \sum_{i=1}^{B} \left\Vert \rvx_{i} - h_{\theta}({\rvz}_{i})\right\Vert_2^2$ 
    \vspace*{0.06cm}
    \State $ \mathcal{L}_{p} \gets \sum_{i=1}^{B} \left\Vert P_{\psi}( \rvx_{i}) - P_{\psi}(h_{\theta}({\rvz}_{i}))\right\Vert_2^2 $
    \vspace*{0.02cm}
    \State $ \mathcal{L}_{g}\gets\mathcal{L}_{r} + \mathcal{L}_{p} + \sum_{i=1}^{B} \mathds{D}_{KL} \left[\mathcal{N}(\mu_\rvz^{(i)},{\sigma_\rvz^{(i)}}^2 ) \,|| \,\mathcal{N}(0, 1) \right]$
    \State $ \mathcal{L}_{h} \gets  \mathcal{L}_{r} + \mathcal{L}_{p}$ 
    \State $ \phi \gets \phi + \eta \nabla_{\phi}\mathcal{L}_{g} $ 
    \State $ \theta \gets \theta + \eta \nabla_{\theta}\mathcal{L}_{h}$ 
    \UNTIL{convergence of $\phi$, $\theta$ }
    \vspace*{.35cm}
    \phase{Inference - Latent Search during testing with Target}
     \textbf{Input}: Target sample $\rm \tilde{\rvx}_\mathcal{T}$, Trained decoder $h_{\theta^{\ast}}$, Learning rate $\eta$. \textbf{Output}: `nearest-clone' $\tilde{{\mathbf{x}}}_\mathcal{S}$  for the target sample $\rm \tilde{\rvx}_\mathcal{T}$.  
    \State sample ${\mathbf{z}}$ from $\mathcal{N}(0,1)$
    \REPEAT
    
    \State $ \mathcal{L}_{ssim} \gets 1 - \text{SSIM}({\rm \tilde{x}}_\mathcal{T}, h_{\theta^*}({\rvz})) $
    \State $ {\rvz} \gets {\rvz}  + \eta \nabla_{\rvz}\mathcal{L}_{ssim} $
    \UNTIL{convergence of $\mathcal{L}_{ssim}$}
    \State $ \rm \tilde{\rvz}_\mathcal{S} \gets {\rvz}$ \State $\tilde{\rvx}_\mathcal{S} \gets h_{\theta^*} (\tilde{\mathbf{z}}_\mathcal{S})$
    \end{algorithmic}
\end{algorithm}
\section{Implementation details}
\subsection{GLSS on NIR images}
$\mathcal{S}_{\psi}$ is the segmentation model (as shown in Fig. 2 in the paper) implemented using DeepLabv3+ (XceptionNet) and UNet (EfficientNet). $\mathcal{S}_{\psi}$ is trained for 100-150 epochs with losses ($\mathcal{L}_{s}$) as shown in Eq. 1 and Eq. 2  for UNet  and DeepLabv3+ respectively.

\begin{equation}
    \mathcal{L}_{s} =  \mathcal{L}_{dice} + \mathcal{L}_{bce}
\end{equation}
\begin{equation}
    \mathcal{L}_{s} = \mathcal{L}_{focal}
\end{equation}

$\mathcal{L}_{dice}$ is the dice coefficient loss which calculates the overlap between the predicted and the ground truth mask whereas $\mathcal{L}_{bce}$ is the binary cross-entropy loss.
Binary focal loss ($\mathcal{L}_{focal}$) tries to down-weight the contribution
of examples that can be easily segmented so that the segmentation model focuses  more on learning hard examples.

$P_{\psi}$ is a perceptual model (as shown in Fig. 1 in the paper) that uses perceptual loss $\mathcal{L}_{p}$. The perceptual features are taken from the 6th layer of UNet and the last concatenation layer of DeepLabv3+. VAE along with perceptual loss $\mathcal{L}_{p}$ is trained for  150-200 epochs. $\mathcal{L}_{p}$ is weighted  with a factor $\beta$ (a hyper-parameter) as shown:
\begin{equation}
    \mathcal{L}_{total}  =  \mathcal{L}_{vae} + \beta \mathcal{L}_{p} 
\end{equation}
In order to improve the quality of VAE reconstructed images, we weighted the perceptual loss ($\mathcal{L}_{p}$) with different values of $\beta$.
For UNet, we have used $\beta$ = 2 whereas  $\beta$ = 3 is used for DeepLabv3+.
The first part of Algorithm 1 shows the steps involved in training VAE and second part shows the steps involved in inference procedure. 

Using an Intel Xeon processor (6 Cores) with a base frequency of 2.0 GHz, 32GB RAM and NVIDIA® Tesla® K40 (12 GB Memory) GPU, Latent Search for one sample on SNV dataset takes 450 ms and 120 ms on Hand Gesture dataset. The time required is in the order of milliseconds on a basic GPU like K40 which is not very significant. However, this is the cost that is to be paid for being target independent which is a very significant advantage.
\subsection{Implementation details of UDA baseline methods for skin segmentation}
DeepLabv3+ was used as the segmentation model for all the baselines with images and corresponding masks of size $128\times128$. AdaptsegNet \cite{tsai2018learning} uses discriminative approach to predict the domain of the images.
For discriminator, we used a model with 5 convolutional layers (default implementation). We performed a grid search over  $\mathcal{\lambda}_{advtarget1}$ and $\mathcal{\lambda}_{advtarget2}$ and reported the best IoU score for AdaptsegNet.
DISE \cite{chang2019all} uses image-to-image translation approach to translate one domain to another. It employs label transfer loss to optimize the segmentation model. Image-to-image translation based methods work well in cases where the structural similarity is more. 
We used 0.1, 0.25 and 0.5 for $\mathcal{\lambda}_{seg}$ and reported the best IoU using $\mathcal{\lambda}_{seg}=0.1$ while the learning rate was set to 2.5e-4.
Advent \cite{vu2018advent} proposes to leverage an entropy loss to directly penalize low-confident predictions on target domain. If $\mathcal{\lambda}_{ent}$ is large then the entropy drops too quickly and the model is strongly biased towards a few classes. We used 0.001 for $\mathcal{\lambda}_{ent}$ as suggested by the authors regardless of the network and dataset. Also, for adversarial training, 0.001 was used for $\mathcal{\lambda}_{adv}$. We trained with AdvEnt as it performed better that minEnt as stated in the paper. SGD and Adam were used as optimizers for segmentation and discriminator networks respectively.
In DADA \cite{vu2019dada}, authors make use of an additional depth information in the source domain. We performed a grid search over $\mathcal{\lambda}_{seg}$ using values 0.25, 0.5, 1. The learning rate was varied with values 2.5e-4, 1e-4 and 3e-4 and finally best IoU was reported with $\mathcal{\lambda}_{seg}=0.5$ and learning rate = 2.5e-4. 
CLAN \cite{luo2019taking} makes use of a category-level joint distribution and align each class with an adaptive adversarial loss, thus ensuring correct mapping of source and target. Compared to traditional adversarial training, CLAN introduces the discrepancy loss and the category-level adversarial loss. Hyperparameters like learning rate, weight decay, $\mathcal{\lambda}_{weight}$ and $\mathcal{\lambda}_{adv}$ were used with values 2.5e-4, 5e-4, 0.01 and 0.001 respectively during training. 
For training BDL \cite{li2019bidirectional}, we set the learning rate to 2.5e-4 for the segmentation network and 1e-4 for the discriminator. Grid search was performed for $\mathcal{\lambda}_{advtarget}$ with values 1e-3, 2e-3, 5e-3 and best IoU was reported with $\mathcal{\lambda}_{advtarget}=$ 1e-3.

\subsection{SSIM Loss}
SSIM loss compares pixels and their corresponding neighbourhoods between two images, preserving the luminance, contrast and structural information. To perform Latent Search, we used distance metric as SSIM loss, that helps to sample the ‘nearest-clone’ in the source distribution for the target image from the generative latent space of VAE. Unlike norm-based losses, SSIM loss helps in the preservation of structural information as compared to discrete pixel-level information. We used 11x11 Gaussian filter in our experiments.

SSIM  is defined using the three aspects of similarities, luminance $\big(l(\mathbf{x}, \hat{\mathbf{x}})\big)$, contrast $\big(c(\mathbf{x}, \hat{\mathbf{x}})\big)$ and structure $\big(s(\mathbf{x}, \hat{\mathbf{x}})\big)$ that are measured for a pair of images $\{\mathbf{x}, \hat{\mathbf{x}}\}$ as follows:
\begin{equation}
l(\mathbf{x}, \hat{\mathbf{x}}) = \frac{2\mu_\mathbf{x}\mu_{\hat{\mathbf{x}}}+C_1}{\mu_\mathbf{x}^2+ \mu_{\hat{\mathbf{x}}}^2 + C_1}
\end{equation}
\begin{equation}
c(\mathbf{x}, \hat{\mathbf{x}}) = \frac{2\sigma_{\mathbf{x}}\sigma_{\hat{\mathbf{x}}}+C_2}{{\sigma_\mathbf{x}}^2+ {\sigma_{\hat{\mathbf{x}}}}^2 + C_2}
\end{equation}
\begin{equation}
s(\mathbf{x}, \hat{\mathbf{x}}) = \frac{\sigma_{\mathbf{x}{\hat{\mathbf{x}}}}+C_3}{\sigma_{\mathbf{x}}\sigma_{\hat{\mathbf{x}}} + C_3}
\end{equation}
where $\mu$'s denote sample means and $\sigma$'s denote variances. $C_1, C_2$ and $C_3$ are constants. With these, SSIM and the corresponding loss function $\mathcal{L}_{ssim}$, for a pair of images $\{\mathbf{x}, \hat{\mathbf{x}}\}$ are defined as: 
\begin{equation}
\text{SSIM}(\mathbf{x}, \hat{\mathbf{x}}) = l(\mathbf{x}, \hat{\mathbf{x}})^{\alpha} \cdot c(\mathbf{x}, \hat{\mathbf{x}})^{\beta} \cdot s(\mathbf{x}, \hat{\mathbf{x}})^{\gamma}  
\end{equation}
where $\alpha>0$, $\beta>0$ and $\gamma>0$ are parameters used to adjust the relative importance of the three components.
\begin{equation}
\mathcal{L}_{ssim}(\mathbf{x}, \hat{\mathbf{x}}) = 1 - \text{SSIM}(\mathbf{x}, \hat{\mathbf{x}})
\end{equation}

\begin{table}[h]
\caption{IoU comparison for Target-Independence of GLSS with change in the amount of target data. GLSS performance is not affected by change in the amount of target data during training while other SOTA methods degrade.}
\begin{center}
\scalebox{0.77}{
\begin{tabular}{c|c|c|c|c|c|c}
    \toprule
    \% of Target data  &  Adaptsegnet & BDL & CLAN& Advent & DADA & GLSS \\
    \midrule
    
    60 & 0.23 & 0.30 & 0.22 & 0.33 & 0.31 & 0.37 \\ 
    40 & 0.22 & 0.26 & 0.22 & 0.29 & 0.28 & 0.37 \\ 
    20 & 0.21 & 0.22 & 0.21 & 0.24 & 0.23 & 0.38 \\ 
    
\bottomrule
\end{tabular}
}
\end{center}
\label{tab:vaevsgans}
\end{table} 

\subsection{Target-Independence of GLSS}
GLSS is a general-purpose target-independent UDA method. For UDA, target independence is a merit since a SINGLE source model can be used across multiple targets. However, even with target data (for VAE training) GLSS doesn’t degrade while SOTA methods do, for skin segmentation on NIR images (Table below compares IoU). 

\subsection{GAN vs. VAE}
GLSS demands a generative model that has both generation and inference capabilities (mapping from latent to data space and vice versa), which is not the case with GANs. This leads to non-convergence of latent search. To validate this, we trained a SOTA BigGAN \cite{brock2018large} on COMPAQ Dataset [21] and performed GLSS. Although GAN had better generation quality (FID of 29.7 with BigGAN vs. 44 with VAE), the final IoU was worse as shown in Table \ref{tab:vaevsgans}.

\begin{table}[h]
\caption{IoU score comparison between BigGAN and VAE when trained on SNV and Hand Gesture datasets. VAE scores better in terms of IoU.}
\begin{center}
\scalebox{0.77}{
\begin{tabular}{c|c|c|c}
    \toprule
    SNV/BigGAN  &  Hand Gesture/BigGAN & SNV/VAE & Hand Gesture/VAE  \\
    \midrule
    
    0.09&0.21&0.38&0.69\\
\bottomrule
\end{tabular}
}
\end{center}
\label{tab:vaevsgans}
\end{table} 

\clearpage
\section{Additional Results}

\begin{figure}
\begin{center}
    \includegraphics[width=0.7\textwidth,height=.74\textwidth]{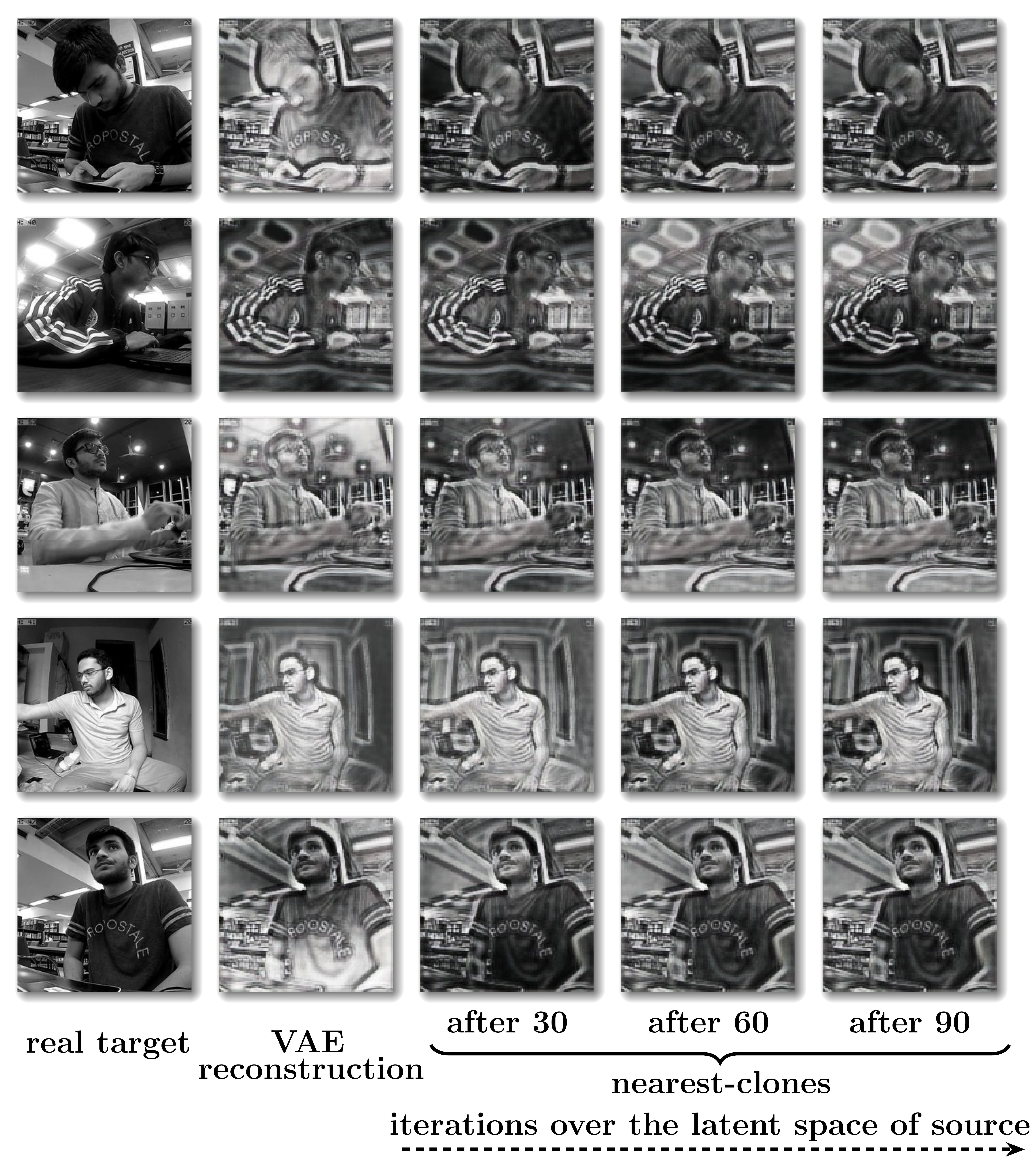}
\caption{Illustration of Latent Search (LS) in GLSS for SNV dataset. Prior to the LS, VAE reconstructed target samples are obtained. It is evident that the `nearest-clones' (images generated using LS) improve as the LS progresses. Also the quality (empirically) of `nearest-clones' are better as compared to the VAE reconstructed images. The `nearest-clones' are shown after every 30 iterations.}
\label{fig:LatentSearch}
\end{center}
\end{figure}
\begin{figure}[h]
\begin{center}
    \includegraphics[width=0.7\textwidth,height=.74\textwidth]{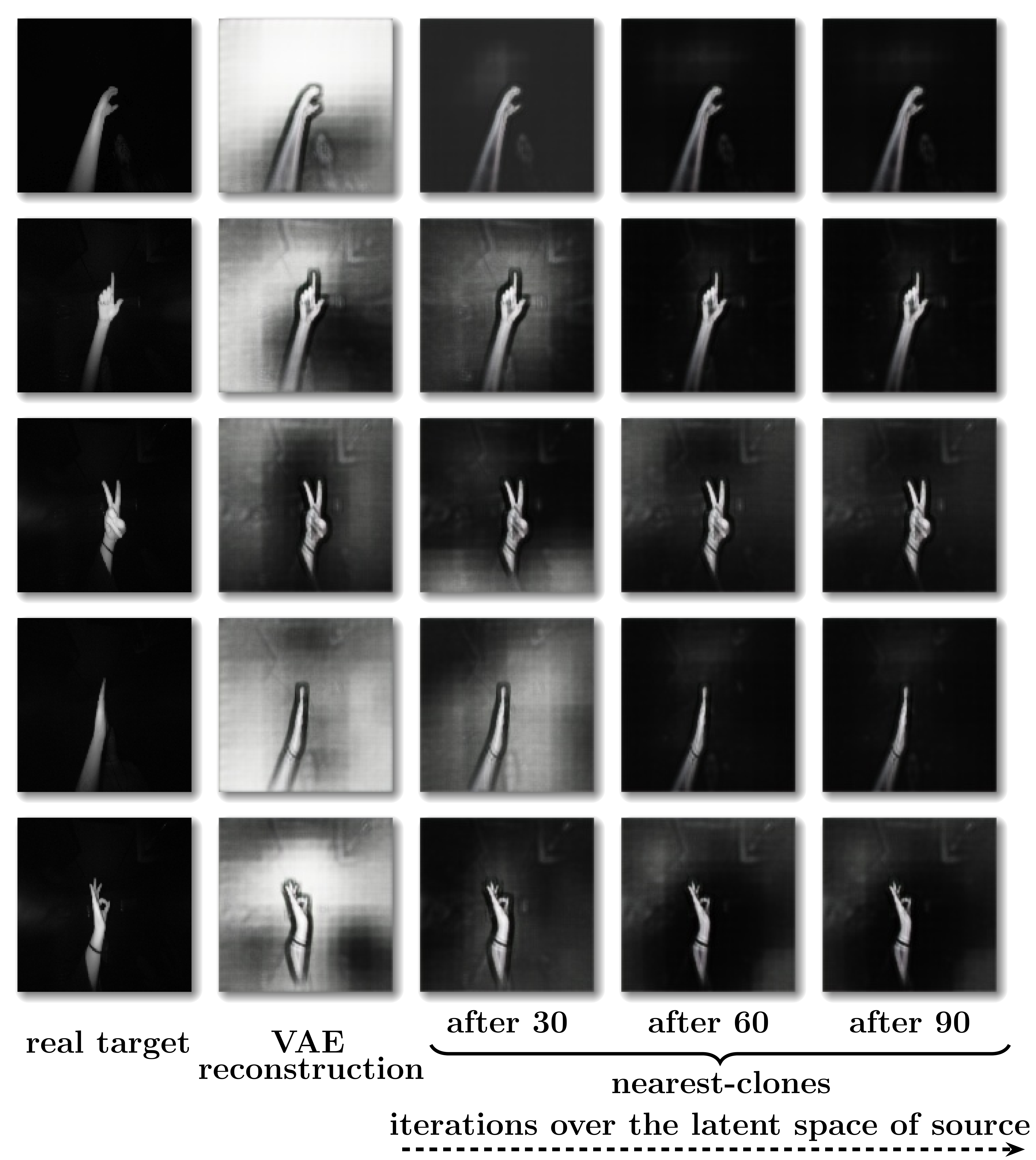}
\caption{Illustration of Latent Search (LS) in GLSS for Hand Gesture dataset. Prior to the LS, VAE reconstructed target samples are obtained. It is evident that the `nearest-clones' (images generated using LS) improve as the LS progresses. Also the quality (empirically) of `nearest-clones' are better as compared to the VAE reconstructed images. The `nearest-clones' are shown after every 30 iterations.}
\label{fig:LatentSearch}
\end{center}
\end{figure}

\begin{figure}[h]
\begin{center}
    \includegraphics[width=0.70\textwidth,height=.32\textwidth]{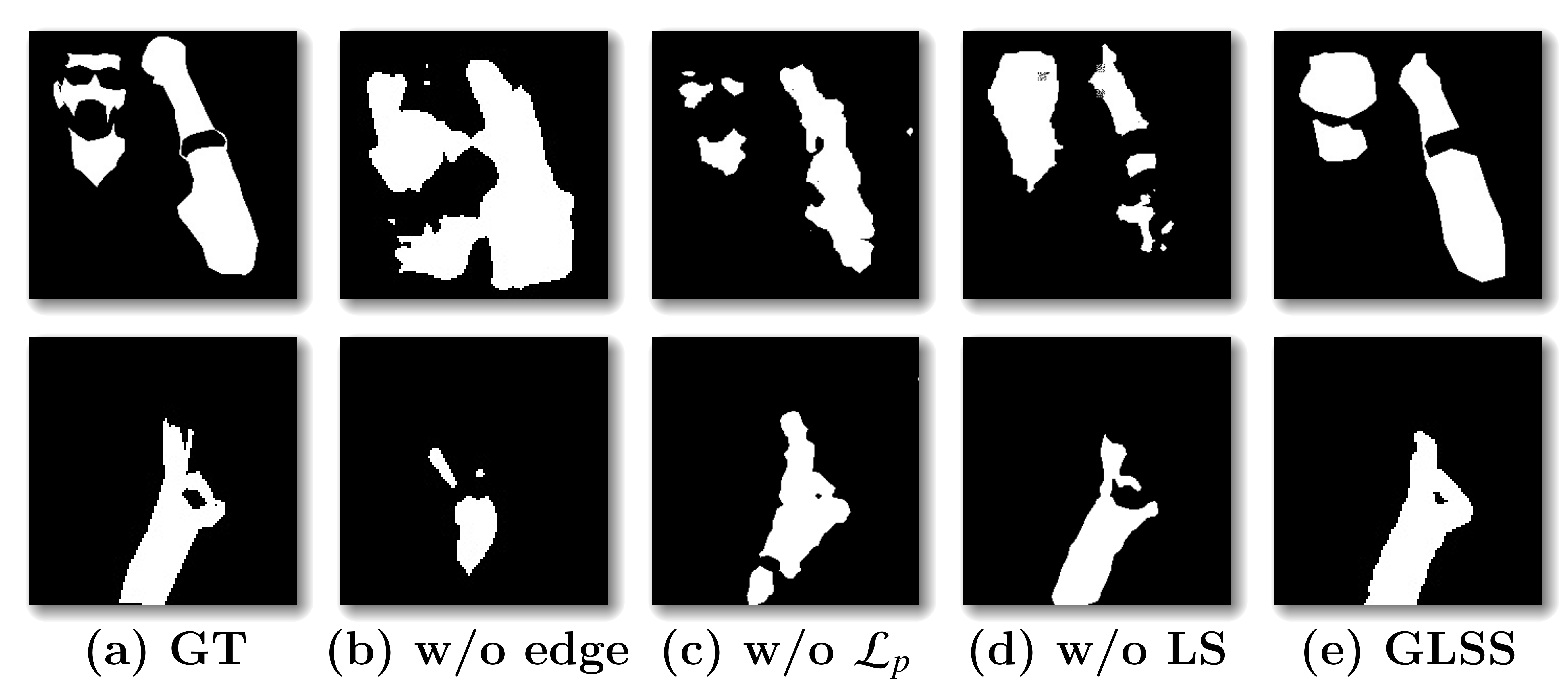}
\caption{(a) the ground truth mask for SNV and Hand Gesture datasets, (b) the predicted mask of VAE reconstructed image without edge concatenation, (c) the predicted mask of VAE reconstructed  image without $\mathcal{L}_{p}$,
(d) the predicted mask of VAE reconstructed with edge concatenation and perceptual loss when no Latent Search (LS) was performed, 
 (e) the predicted mask with GLSS.
 It is evident from the predicted masks that with edge concatenation, perceptual loss and Latent Search (LS), quality of predicted masks improve. Each component plays a significant role in improving the IoU. Hence, when all the components are employed (as in GLSS) we get the best IoU.}
\label{fig:LatentSearch}
\end{center}
\end{figure}
\begin{figure}[h]
\begin{center}
    \hspace*{-1cm}\includegraphics[scale=0.48]{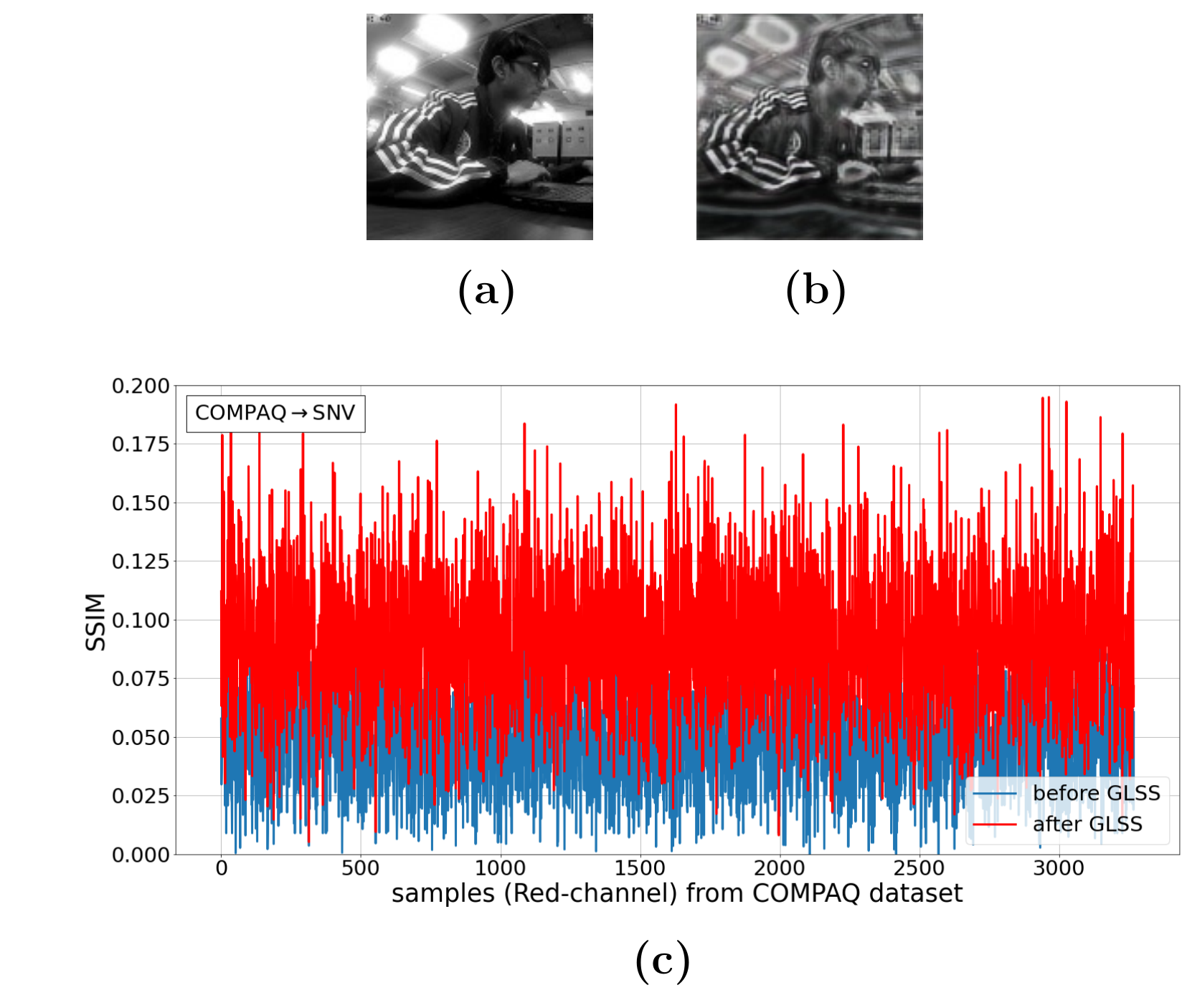}
\caption{(a) an NIR image $ \tilde{\mathbf{x}}_\mathcal{T}$ from SNV dataset (target), (b) `nearest-clone' $ \tilde{\mathbf{x}}_\mathcal{S}$ generated from GLSS, (c) Structural Similarity Index (SSIM) scores calculated between $ \tilde{\mathbf{x}}_\mathcal{T}$ and all the samples (having only Red-channel) of COMPAQ dataset (source) are shown with blue color in the plot. Similary, SSIM scores calculated between $ \tilde{\mathbf{x}}_\mathcal{S}$ and all the samples (having only Red-channel) of COMPAQ dataset are shown with red color. It is evident from the figure that the SSIM scores are higher for the `nearest-clone' $\tilde{\mathbf{x}}_\mathcal{S}$ as compared to the scores with $\tilde{\mathbf{x}}_\mathcal{T}$. It indicates that $\tilde{\mathbf{x}}_\mathcal{S}$ is more closer to the source domain (COMPAQ) as compared to $\tilde{\mathbf{x}}_\mathcal{T}$. Hence, the `nearest-clone' $\tilde{\mathbf{x}}_\mathcal{S}$ generated by GLSS for target $\tilde{\mathbf{x}}_\mathcal{T}$ is used as a proxy in the segmentation network $S_{\psi}$ which is trained only on COMPAQ dataset, thereby increasing the IoU for $\tilde{\mathbf{x}}_\mathcal{T}$.}
 \label{fig:LatentSearch}
\end{center}
\end{figure}
\begin{figure}[h]
\begin{center}
    \hspace*{-2cm}\includegraphics[scale=0.60]{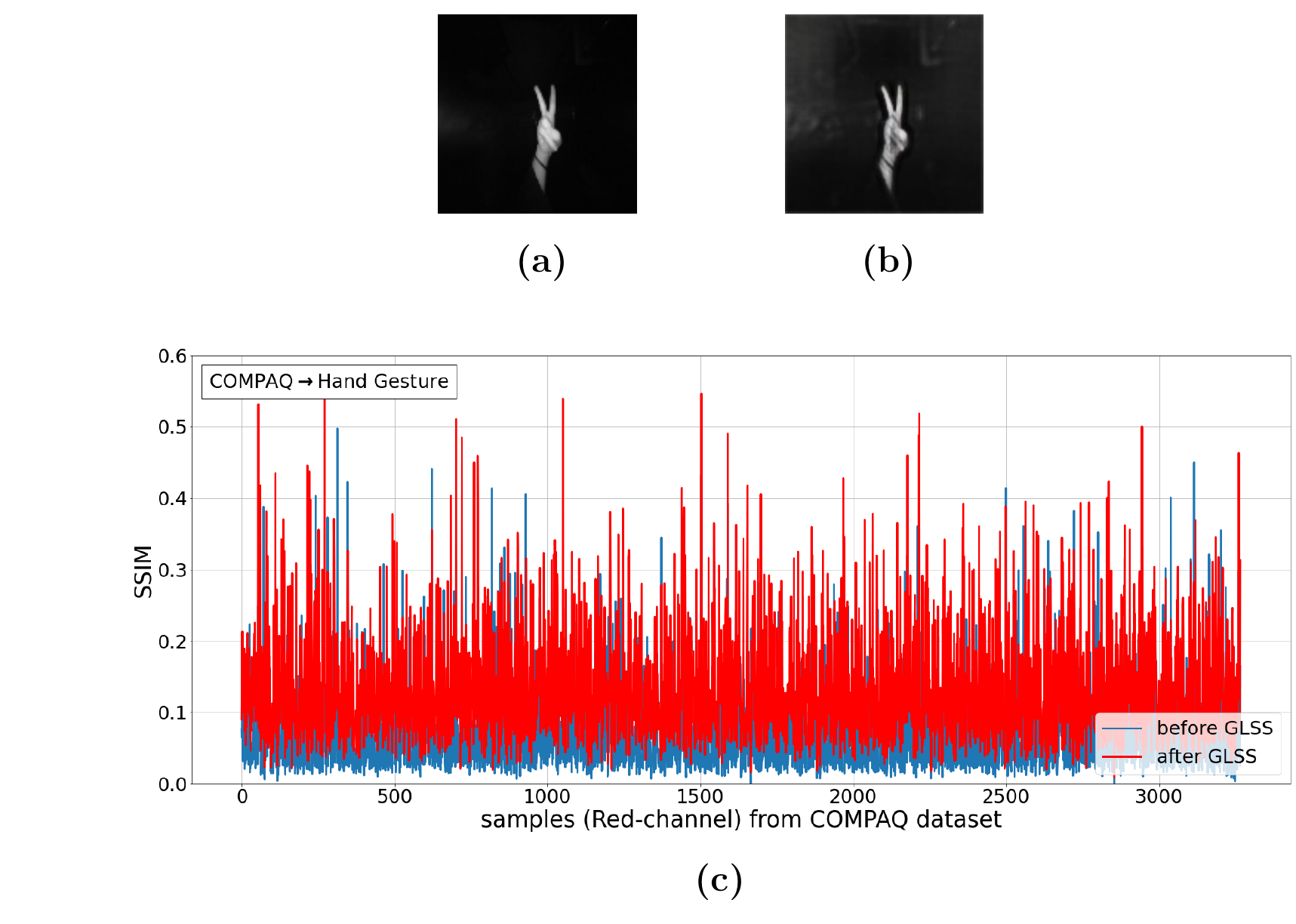}
\caption{(a) an NIR image $ \tilde{\mathbf{x}}_\mathcal{T}$ from Hand Gesture dataset (target), (b) `nearest-clone' $ \tilde{\mathbf{x}}_\mathcal{S}$ generated from GLSS, (c) Structural Similarity Index (SSIM) scores calculated between $ \tilde{\mathbf{x}}_\mathcal{T}$ and all the samples (having only Red-channel) of COMPAQ dataset (source) are shown with blue color in the plot. Similary, SSIM scores calculated between $ \tilde{\mathbf{x}}_\mathcal{S}$ and all the samples (having only Red-channel) of COMPAQ dataset are shown with red color. It is evident from the figure that the SSIM scores are higher for the `nearest-clone' $\tilde{\mathbf{x}}_\mathcal{S}$ as compared to the scores with $\tilde{\mathbf{x}}_\mathcal{T}$. It indicates that $\tilde{\mathbf{x}}_\mathcal{S}$ is more closer to the source domain (COMPAQ) as compared to $\tilde{\mathbf{x}}_\mathcal{T}$. Hence, the `nearest-clone' $\tilde{\mathbf{x}}_\mathcal{S}$ generated by GLSS for target $\tilde{\mathbf{x}}_\mathcal{T}$ is used as a proxy in the segmentation network $S_{\psi}$ which is trained only on COMPAQ dataset, thereby increasing the IoU for $\tilde{\mathbf{x}}_\mathcal{T}$.}
 
 \label{fig:LatentSearch}
\end{center}
\end{figure}

\end{document}